\documentclass[reqno]{amsart}

\usepackage{amssymb}
\usepackage{graphicx}
\usepackage{hyperref}

\DeclareMathOperator{\tr}{tr}

\theoremstyle{plain}
\newtheorem{lemma}{Lemma}

\begin{document}

\title[Relative Pose Estimation for Cameras and Generalized Cameras]{Efficient Relative Pose Estimation for Cameras and Generalized Cameras in Case of Known Relative Rotation Angle}

\author{Evgeniy Martyushev}
\address{South Ural State University, 76 Lenin Avenue, Chelyabinsk 454080, Russia}
\email{mev@susu.ac.ru}

\author{Bo Li}
\address{TrunkTech Co., Ltd., China}
\email{prclibo@gmail.com}

\date{January 31, 2019}

\keywords{Multiview geometry, Relative pose estimation, Generalized cameras, Epipolar constraints, Relative rotation angle, Gr\"{o}bner basis}

\begin{abstract}
We propose two minimal solutions to the problem of relative pose estimation of (i) a calibrated camera from four points in two views and (ii) a calibrated generalized camera from five points in two views. In both cases, the relative rotation angle between the views is assumed to be known. In practice, such angle can be derived from the readings of a 3d gyroscope. We represent the rotation part of the motion in terms of unit quaternions in order to construct polynomial equations encoding the epipolar constraints. The Gr\"{o}bner basis technique is then used to efficiently derive the solutions. Our first solver for regular cameras significantly improves the existing state-of-the-art solution. The second solver for generalized cameras is novel.

The presented minimal solvers can be used in a hypothesize-and-test architecture such as RANSAC for reliable pose estimation. Experiments on synthetic and real datasets confirm that our algorithms are numerically stable, fast and robust.
\end{abstract}

\maketitle

\section{Introduction}

The problem of \emph{relative pose estimation} of a moving camera consists in determining the current camera pose (both position and orientation) with respect to a coordinate frame related to its previous position. Basically, the estimation must be done only from the image data captured by the camera. The relative pose estimation is a central task in computer vision and robotics. Its numerous applications include, but are not limited to, robot localization and mapping, augmented reality, autonomous driving and parking, visual odometry and egomotion estimation.

The standard tool for estimating relative pose of a calibrated camera moving in space is the 5-point algorithm by Nist{\'e}r~\cite{Nister} and its numerous modifications~\cite{SEN06,KBP,LiHartley}. The 5-point algorithm is known to be minimal, i.e. its related polynomial ideal is generically zero-dimensional. There also exist non-minimal relative pose solutions based on a bigger number of point correspondences, e.g. 6-point~\cite{philip98,pizarro03}, 7-point~\cite{HZ}, and linear $8$-point algorithms~\cite{longuet81,hartley95}. However, in most cases these solutions are inferior to the 5-point algorithm w.r.t. the robustness~\cite{Nister,SEN06}.

The 5-point relative pose problem, as well as many other problems in computer vision, can be formulated in terms of polynomial equations on some parameters related to the problem. An efficient method to find solutions of polynomial equations is the \emph{Gr\"{o}bner basis} computation. From linear algebra point of view, the Gr\"{o}bner basis can be computed by a single Gauss-Jordan (G-J) elimination on a certain matrix constructed from the coefficients of polynomials. This matrix is known as an \emph{elimination template} for the Gr\"{o}bner basis computation~\cite{KBP08}, since generically its size and structure are independent from the coefficients of polynomials. The size of such template is crucial for the efficiency of the resulting solver. Hence, it is important to formulate and parameterize the problem so that its elimination template would be as small as possible.

The problem of relative pose estimation can be modified and generalized in many directions. For example, we can consider a similar problem for the so-called \emph{generalized cameras}. In contrast to central projection model of a regular\footnote{Hereinafter, the cameras will sometimes be referred to as regular in order to distinguish them from generalized cameras.} camera, the image rays for a generalized camera do not necessarily intersect in a common focal center of projection~\cite{Pless}. Generalized cameras can be used for describing different popular camera models~\cite{sturm05} such as a multi-camera rig consisting of several rigidly mounted cameras moving as a whole, a non-central catadioptric camera, a camera with radial and tangential distortion, etc. In general, the relative pose of a generalized camera can be estimated minimally from six~\cite{SNOA,VAL15} and linearly from seventeen~\cite{Pless,LHK08} point correspondences in two views.

Another modification of the relative pose estimation problem arises if a camera is coupled with an \emph{inertial measurement unit} (IMU) sensor. The additional information extracted from the IMU readings can be used to significantly simplify the relative pose estimation process or to make it faster and more robust. For example, the IMU measurements can be used to derive the vertical direction in the camera coordinate frame reducing thus the number of unknown extrinsics from five to three~\cite{Kalantari-vert}. See also~\cite{FTP10,KBP10,Lee-vert} for some other camera-IMU fusion applications.

The contribution of the present paper is two new minimal solutions to the problem of relative pose estimation of (i) a calibrated camera from four points in two views and (ii) a calibrated generalized camera from five points in two views. In both cases, the relative rotation angle between the views is assumed to be known. In practice, this angle can be calculated from the readings of an IMU sensor. The advantage in using only relative rotation angle follows from its invariance. As it was mentioned in~\cite{LHLP}, if a camera and IMU are fixed on some rigid platform, then the rotation angle of the IMU can be directly used as the rotation angle of the camera. Thus, there is no need in the external camera-IMU calibration and the joint usage of these devices becomes convenient and flexible in practice.

The problem of pose estimation for regular cameras in case of known relative rotation angle was stated in~\cite{LHLP}, where the two solutions to the problem have been proposed. The first Gr\"{o}bner-basis solution has an elimination template of size $270\times 290$ and hence its practical use is questionable especially for real-time applications. In contrast, our new solver has an elimination template of size $16\times 36$ only. Actually, the Gr\"{o}bner basis computation for our solver is performed even faster than for the 5-point algorithm by Stew{\'e}nius et al., cf. Subsection~\ref{ssec:timings}.

The numerical solution from~\cite{LHLP} is fast enough, although it is still approximately twice slower than our solver. The relative pose of a camera is estimated by the gradient descent-based method from a big enough number of seeds uniformly distributed on a unit 2-sphere. The method works well while the roots on the sphere are well-separated. If it is not the case, there is no guaranty that all solutions to the problem will be found.

As for the pose estimation of generalized cameras in case of known relative rotation angle, the problem has not been studied earlier. Our ``5-point+angle'' solution can be considered as an alternative to the 6-point method from~\cite{SNOA}. Based on the experimental results, we conclude that our solution outperforms the 6-point solver both in numerical accuracy and in speed.

The present paper is a continuation of the research started in recent work~\cite{Mart18}, where a new method for the camera self-calibration has been proposed. The method is also based on the joint usage of multiple point correspondences in two views and the known relative rotation angle.

The rest of the paper is organized as follows. In Section~\ref{sec:prel}, we briefly recall some notions from multiview geometry, e.g. the epipolar and generalized epipolar constraints which are needed to follow the derivation of our minimal solvers. In Section~\ref{sec:alg4}, we describe in detail the 4-point algorithm for regular cameras. Section~\ref{sec:alg5} summarizes the similar 5-point approach for generalized cameras. In Section~\ref{sec:synth}, we test the algorithms in a series of experiments on synthetic data. In Section~\ref{sec:real}, the algorithms are validated on real datasets. In Section~\ref{sec:disc}, we discuss the results of the paper.

\section{Preliminaries}
\label{sec:prel}

\subsection{Notation}

We preferably use $\alpha, \beta, \ldots$ for scalars, $a, b, \ldots$ for column 3-vectors or polynomials, and $A, B, \ldots$ both for matrices and column 4-vectors. For a matrix $A$ the entries are~$(A)_{ij}$, the transpose is~$A^{\mathrm T}$, the determinant is $\det A$, and the trace is~$\tr A$. For two 3-vectors~$a$ and~$b$ the cross product is $a\times b$. For a vector~$a$ the notation~$[a]_\times$ stands for the skew-symmetric matrix such that $[a]_\times b = a \times b$ for any vector~$b$. We use~$I$ for the identity matrix and $\|\cdot\|$ for the Frobenius norm.

\subsection{Regular Cameras}

Let there be given two calibrated cameras represented by $3\times 4$ matrices of the form
\begin{equation}
P' = \begin{bmatrix} R' & t' \end{bmatrix}, \quad P'' = \begin{bmatrix} R'' & t'' \end{bmatrix},
\end{equation}
where $R', R'' \in \mathrm{SO}(3)$ are called the \emph{rotation matrices} and~$t', t''$ are called the \emph{translation vectors}. Let~$Q$ be a $4$-vector representing a point in 3-space in homogeneous coordinates, $q'$ and~$q''$ be its images, that is
\begin{equation}
q' \sim P' Q, \quad q'' \sim P'' Q,
\end{equation}
where $\sim$ means an equality up to scale. Then the \emph{epipolar constraint} for a pair $(q', q'')$ reads
\begin{equation}
\label{eq:epipolar}
q''^{\mathrm T} (R [t']_\times - [t'']_\times R) q' = 0,
\end{equation}
where $R = R''R'^{\mathrm T}$ is the \emph{relative rotation matrix}, i.e. a rotation matrix of the second camera w.r.t. the first camera frame.

\subsection{Generalized Cameras}

Given two generalized cameras, let $(R', t')$ (resp. $(R'', t'')$) be the rotation and translation of the first (resp. second) generalized camera w.r.t. the world frame. Let $\begin{bmatrix}q'& m'\end{bmatrix}^{\mathrm T}$ and $\begin{bmatrix}q''& m''\end{bmatrix}^{\mathrm T}$ be the Pl\"{u}cker coordinates of a pair of image rays intersecting at a point. Transforming coordinates $\begin{bmatrix}q'& m'\end{bmatrix}^{\mathrm T}$ into the world coordinate frame yields~\cite{SNOA}
\begin{equation}
\begin{bmatrix}\hat q'\\ \hat m'\end{bmatrix} = \begin{bmatrix}R'^{\mathrm T}q'\\ R'^{\mathrm T}m' + R'^{\mathrm T}[t']_\times q'\end{bmatrix},
\end{equation}
and similarly for $\begin{bmatrix}\hat q'' & \hat m''\end{bmatrix}^{\mathrm T}$. Then the intersection condition $\hat q'^{\mathrm T} \hat m'' + \hat q''^{\mathrm T} \hat m' = 0$ leads to the \emph{generalized epipolar constraint}~\cite{Pless,SNOA}
\begin{equation}
\label{eq:epipolar_gen}
\begin{bmatrix}q''\\ m''\end{bmatrix}^{\mathrm T}\begin{bmatrix}R [t']_\times - [t'']_\times R & R\\ R & 0_{3\times 3}\end{bmatrix} \begin{bmatrix}q'\\ m'\end{bmatrix} = 0,
\end{equation}
where $R = R''R'^{\mathrm T}$ is the relative rotation matrix.

\subsection{Quaternion Representation}

A matrix $R \in \mathrm{SO}(3)$ can be represented in terms of a unit quaternion $\begin{bmatrix}\sigma & u\end{bmatrix}^{\mathrm T}$ as follows
\begin{equation}
\label{eq:matrixR}
R = (2\sigma^2 - 1)I + 2(uu^{\mathrm T} - \sigma[u]_\times),
\end{equation}
where
\begin{equation}
\label{eq:quadconstr}
\|u\|^2 + \sigma^2 = 1.
\end{equation}
Then we have
\begin{equation}
\tr R = 4\sigma^2 - 1 = 2\cos\theta + 1,
\end{equation}
where $\theta$ is the relative rotation angle between the two calibrated camera frames. It follows that in case~$\theta$ is known, the value of parameter~$\sigma$ in representation~\eqref{eq:matrixR} is known too and can be expressed in form
\begin{equation}
\label{eq:sigma2}
\sigma^2 = \frac{\cos\theta + 1}{2}.
\end{equation}
The two-fold ambiguity $\pm\begin{bmatrix}\sigma & u\end{bmatrix}^{\mathrm T}$ of the quaternion representation is resolved by fixing the sign of parameter~$\sigma$.

\section{Description of the 4-point Algorithm}
\label{sec:alg4}

The initial data is the four point correspondences $q'_i \leftrightarrow q''_i$, $i = 1, \ldots, 4$, and also the relative rotation angle~$\theta$.

\subsection{Polynomial Equations}

We start from deriving polynomial constraints on the camera extrinsics using the parametrization of~$t'$ and~$t''$ adapted from~\cite{SNOA} for regular cameras. Let~$Q_i$ be a point in 3-space so that
\begin{equation}
\lambda_i q'_i = P' Q_i, \quad \mu_i q''_i = P'' Q_i,
\end{equation}
where $\lambda_i$ and~$\mu_i$ are some scalars. Using the rigid motion ambiguity of the world coordinate frame, we set $Q_i = \begin{bmatrix} 0 & 0 & 0 & 1\end{bmatrix}^{\mathrm T}$ for a certain~$i \in \{1, \ldots, 4\}$. Then it follows that
\begin{equation}
\label{eq:transl}
t' = \lambda_i q'_i, \quad t'' = \mu_i q''_i.
\end{equation}

The epipolar constraint~\eqref{eq:epipolar} for the $j$th and $k$th correspondences yields
\begin{equation}
\label{eq:eqF}
F_{ijk} \begin{bmatrix}\lambda_i \\ \mu_i\end{bmatrix} = 0_{2\times 1},
\end{equation}
where
\begin{equation}
\label{eq:Fijk}
F_{ijk} = \begin{bmatrix} {q''_j}^{\mathrm T} R p'_{ij} & {p''_{ij}}^{\mathrm T} R q'_j \\ {q''_k}^{\mathrm T} R p'_{ik} & {p''_{ik}}^{\mathrm T} R q'_k \end{bmatrix}.
\end{equation}
Here vector $p'_{ij}$ is defined by $p'_{ij} = q'_i \times q'_j$ and similarly for $p''_{ij}$.

Eq.~\eqref{eq:eqF} assumes that $\det F_{ijk} = 0$ for all (pairwise distinct) indices $i, j$ and~$k$. There exist 20 possible equations, but due to the following lemma, only 4 of them are actually distinct.

\begin{lemma}
\label{lem:FF}
Let $F_{ijk}$ be a matrix defined in~\eqref{eq:Fijk}. Then the following identity holds
\begin{equation}
\label{eq:FF}
\det F_{ijk} = \det F_{jki}.
\end{equation}
\end{lemma}

The proof of Lemma~\ref{lem:FF} is given in Appendix~\ref{app:A}. We represent the relative rotation matrix~$R$ in form~\eqref{eq:matrixR} and then, based on the lemma, define four quartic polynomials
\begin{equation}
\label{eq:4pols}
\begin{split}
f_1 &= \det F_{234} = \det F_{342} = \det F_{423},\\
f_2 &= \det F_{341} = \det F_{413} = \det F_{134},\\
f_3 &= \det F_{412} = \det F_{124} = \det F_{241},\\
f_4 &= \det F_{123} = \det F_{231} = \det F_{312}.
\end{split}
\end{equation}

Further, let
\begin{equation}
\label{eq:polh}
h = \|u\|^2 + \sigma^2 - 1,
\end{equation}
where parameter~$\sigma$ is expressed in terms of the known angle~$\theta$ by formula~\eqref{eq:sigma2}. Thus we have the following system of polynomial equations
\begin{equation}
\label{eq:sys}
f_1 = f_2 = f_3 = f_4 = h = 0.
\end{equation}
The objective is to find all its real solutions.

\subsection{Gr\"{o}bner Basis}
\label{ssect:gb}

Let us set $u = \begin{bmatrix} \alpha & \beta & \gamma\end{bmatrix}^{\mathrm T}$ and define an ideal
\begin{equation}
J = \langle f_1, \ldots, f_4, h \rangle \subset \mathbb C[\alpha, \beta, \gamma].
\end{equation}
In order to find out the structure of the affine variety of~$J$, we generated a number of random instances of the problem over the field of rationals. Then, using the computer algebra system Macaulay2~\cite{macaulay}, we found that ideal~$J$ is generically zero-dimensional of degree~$20$, that is in general system~\eqref{eq:sys} has exactly 20 complex solutions. The Gr\"{o}bner basis of~$J$ w.r.t. the graded reverse lexicographic order consists of ten polynomials with the following leading monomials:
\begin{equation}
\beta^2\gamma^3 \quad \alpha\gamma^4 \quad \beta\gamma^4 \quad \gamma^5 \quad \alpha\beta^3 \quad \beta^4 \quad \alpha\beta^2\gamma \quad \beta^3\gamma \quad \alpha\beta\gamma^2 \quad \alpha^2.
\end{equation}

Now we propose an efficient algorithm for computing the Gr\"{o}bner basis of ideal~$J$.

Let matrix~$A$ of size $4\times 35$ correspond to the polynomials $f_1, \ldots, f_4$ defined in~\eqref{eq:4pols}, i.e.
\begin{equation}
A x = 0_{4\times 1},
\end{equation}
where $x$ is a monomial vector. The $i$th row of matrix~$A$ consists of coefficients of the $i$th polynomial.

First we expand system~\eqref{eq:sys} with 12 more polynomials $\alpha f_i$, $\beta f_i$ and~$\gamma f_i$ for $i = 1, \ldots, 4$, and 20 polynomials of the form $m h$ for~$m$ being every monomial of the total degree up to~$3$. Thus we get
\begin{equation}
\label{eq:hatAhatx}
\hat A \hat x = 0_{36\times 1},
\end{equation}
where $\hat A$ is the $36\times 56$ coefficient matrix and~$\hat x$ is a monomial vector consisting of all up to degree~5 monomials. Let the first 20 monomials in~$\hat x$ be
\begin{multline}
\alpha^5 \quad \alpha^4\beta \quad \alpha^3\beta^2 \quad \alpha^2\beta^3 \quad \alpha^4\gamma \quad \alpha^3\beta\gamma \quad \alpha^2\beta^2\gamma \quad \alpha^3\gamma^2 \quad \alpha^2\beta\gamma^2 \quad \alpha^2\gamma^3 \\ \alpha^4 \quad \alpha^3\beta \quad \alpha^2\beta^2 \quad \alpha^3\gamma \quad \alpha^2\beta\gamma \quad \alpha^2\gamma^2 \quad \alpha^3 \quad \alpha^2\beta \quad \alpha^2\gamma \quad \alpha^2,
\end{multline}
and the rest of monomials in~$\hat x$ be ordered w.r.t. the graded reverse lexicographic order with $\alpha > \beta > \gamma$. Then
matrix~$\hat A$ can be represented in the following block form
\begin{equation}
\label{eq:matrixhatA}
\hat A = \begin{bmatrix}U & V \\ W & X\end{bmatrix},
\end{equation}
where $U$ is an upper-triangular $20\times 20$ matrix with $1$'s on its main diagonal, see Fig.~\ref{fig:matrixA}. It follows that $\det U = 1$ and the inverse to~$U$ always exists. By elementary row operations, matrix~$\hat A$ is equivalent to
\begin{equation}
\label{eq:newmatrix}
\begin{bmatrix}U & V \\ 0_{16\times 20} & X - WU^{-1}V\end{bmatrix}.
\end{equation}

\begin{figure}[t]
\centering
\includegraphics[width=0.8\hsize]{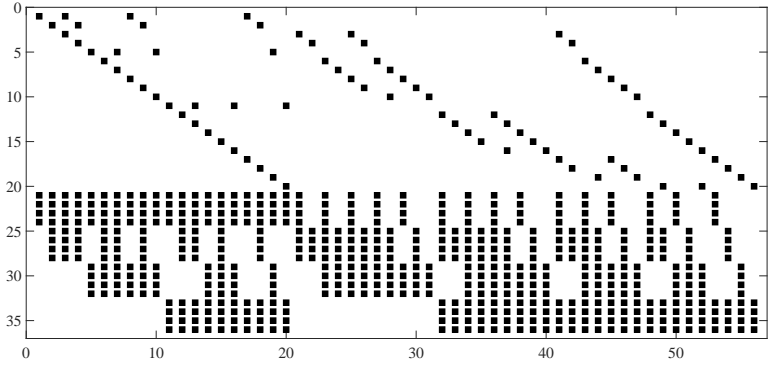}
\caption{The sparse structure of $36\times 56$ matrix~$\hat A$. The main diagonal of the $20\times 20$ left upper submatrix of~$\hat A$ consists of $1$'s}\label{fig:matrixA}
\end{figure}

Matrix $B = X - WU^{-1}V$ of size $16\times 36$ contains all necessary data for the Gr\"{o}bner basis computation. Its structure is shown in Fig.~\ref{fig:matrixB}.

It is important to note that matrix~$B$ needs not to be computed by its definition in an implementation. Instead, it is more efficient to use the pre-computed formulas for the nonzero entries of~$B$. These formulas are quite simple and contain the entries of matrix~$A$ and also the scalar $\tau = \sigma^2 - 1$. For example, one of the most ``complicated'' formulas reads
\begin{equation}
(B)_{16, 34} = 2\tau (A)_{4, 1} - \tau (A)_{4, 10} - (A)_{4, 26} + (A)_{4, 31}.
\end{equation}

\begin{figure}[t]
\centering
\includegraphics[width=0.4\hsize]{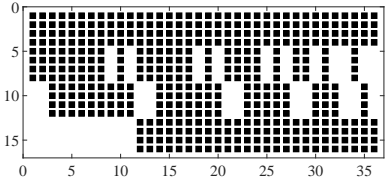}
\caption{The final $16\times 36$ elimination template matrix for computing the Gr\"{o}bner basis of~$J$}\label{fig:matrixB}
\end{figure}

Let~$\tilde B$ be the reduced row echelon form of~$B$, that is, generically,
\begin{equation}
\tilde B = \begin{bmatrix}I & C\end{bmatrix},
\end{equation}
where matrix~$C$ is of size $16\times 20$. Then, the polynomials corresponding to the last nine rows of matrix~$\tilde B$ together with polynomial~$h$ constitute the Gr\"{o}bner basis of ideal~$J$.

\subsection{Relative Rotation and Translation Recovery}

Given the Gr\"{o}bner basis of~$J$, we can readily construct the $20\times 20$ action matrix~$M_\gamma$ for multiplication by~$\gamma$ in the quotient ring $\mathbb C[\alpha, \beta, \gamma]/J$. The first seven rows of~$M_\gamma$ are the rows of matrix~$(-C)$ with numbers $8, 9, 10, 11, 14, 15$ and~$16$. The rest of~$M_\gamma$ consists of almost all zeros except
\begin{multline}
(M_\gamma)_{8,1} = (M_\gamma)_{9,2} = (M_\gamma)_{10,3} = (M_\gamma)_{11,4}\\
= (M_\gamma)_{12,7} = (M_\gamma)_{13,8} = (M_\gamma)_{14,9} = (M_\gamma)_{15,10} = (M_\gamma)_{16,11}\\
= (M_\gamma)_{17,14} = (M_\gamma)_{18,15} = (M_\gamma)_{19,16} = (M_\gamma)_{20,19} = 1.
\end{multline}
Up to 20 real solutions for vector~$u$ are then found from the eigenvectors of matrix~$M_\gamma$, see~\cite{CLS} for details.

It is important to note that because of numerical inaccuracies the solution for~$u$ does not exactly satisfy Eq.~\eqref{eq:quadconstr}. In order to rectify the solution, we replace~$u$ with the vector
\begin{equation}
\hat u = \frac{\sqrt{1 - \sigma^2}}{\|u\|}\, u.
\end{equation}
Then, rotation matrix~$R$ is computed from the unit quaternion $\begin{bmatrix}\sigma & \hat u\end{bmatrix}^{\mathrm T}$ by formula~\eqref{eq:matrixR}.

Using the rigid motion ambiguity of the world coordinate frame, we now set $t' = 0_{3\times 1}$. Then, translation vector~$t''$ is found from the epipolar constraints as the right null-vector of the matrix
\begin{equation}
\begin{bmatrix}
{q'_1}^{\mathrm T} R^{\mathrm T} [q''_1]_\times\\
{q'_2}^{\mathrm T} R^{\mathrm T} [q''_2]_\times\\
{q'_3}^{\mathrm T} R^{\mathrm T} [q''_3]_\times\\
{q'_4}^{\mathrm T} R^{\mathrm T} [q''_4]_\times
\end{bmatrix}.
\end{equation}
The scale ambiguity allows us to set $\|t''\| = 1$. Finally, the sign of~$t''$ is disambiguated by means of the cheirality constraint, see~\cite{HZ,Nister} for details.

\section{The 5-point Algorithm for Generalized Cameras}
\label{sec:alg5}

The initial data is the five corresponding image rays given in Pl\"{u}cker coordinates $\begin{bmatrix}q'_i \\ m'_i \end{bmatrix} \leftrightarrow \begin{bmatrix}q''_i \\ m''_i \end{bmatrix}$, $i = 1, \ldots, 5$, and also the relative rotation angle~$\theta$.

\smallskip
First, we fix the scale of the Pl\"{u}cker coordinates by setting $\|q'\| = \|q''\| = 1$. Then, for~$Q_i$ being a point in 3-space, it follows that
\begin{equation}
\label{eq:projgen}
[m'_i]_\times q'_i + \lambda_i q'_i = \begin{bmatrix}R' & t' \end{bmatrix} Q_i, \quad [m''_i]_\times q''_i + \mu_i q''_i = \begin{bmatrix}R'' & t'' \end{bmatrix} Q_i,
\end{equation}
where $\lambda_i$ and~$\mu_i$ are some scalars. We again follow~\cite{SNOA} to parameterize translations~$t'$ and~$t''$ by setting $Q_i = \begin{bmatrix} 0 & 0 & 0 & 1\end{bmatrix}^{\mathrm T}$ for a certain $i \in \{1, \ldots, 5\}$. Then it can be readily seen from~\eqref{eq:projgen} that
\begin{equation}
\label{eq:transl_gen}
t' = [m'_i]_\times q'_i + \lambda_i q'_i, \quad t'' = [m''_i]_\times q''_i + \mu_i q''_i.
\end{equation}
Substituting this into the generalized epipolar constraint~\eqref{eq:epipolar_gen} for the $j$th, $k$th and $l$th correspondences yields
\begin{equation}
\label{eq:eqF_gen}
G_{ijkl} \begin{bmatrix}\lambda_i \\ \mu_i \\ 1\end{bmatrix} = 0_{3\times 1},
\end{equation}
where $G_{ijkl}$ are $3\times 3$ matrices depending on~$R$ and the initial data. Eq.~\eqref{eq:eqF_gen} assumes that $\det G_{ijkl} = 0$ for all (pairwise distinct) indices $i, j, k$ and~$l$. There exist 20 possible equations but only the following 5 of them are actually distinct:
\begin{equation}
\label{eq:5pols}
\begin{split}
g_1 &= \det G_{2345},\\
g_2 &= \det G_{3451},\\
g_3 &= \det G_{4512},\\
g_4 &= \det G_{5123},\\
g_5 &= \det G_{1234}.
\end{split}
\end{equation}
If the relative rotation matrix~$R$ is represented in form~\eqref{eq:matrixR}, then all polynomials~$g_i$ are sextic in the entries of $u = \begin{bmatrix} \alpha & \beta & \gamma\end{bmatrix}^{\mathrm T}$. The last quadratic polynomial~$h$ is defined in~\eqref{eq:polh}.

In order to efficiently compute the Gr\"{o}bner basis of ideal $\langle g_1, \ldots, g_5, h \rangle$, we use the same approach as in Subsection~\ref{ssect:gb}. We start from the $134\times 165$ elimination template matrix~$\hat A$ corresponding to all possible multiples of the initial polynomials up to the total degree~$8$. Then, carefully eliminating redundant polynomials (rows) from~$\hat A$, we reduced its size to $121\times 165$. The rows of the resulting matrix are represented by
\begin{itemize}
\item $84$ polynomials of the form $m h$, where~$m$ runs over the set of all monomials of the total degree up to~$6$;
\item $35$ polynomials of the form $n g_i$, where $i = 1, \ldots, 5$ and~$n$ belongs to the set $\{\beta^2, \alpha\gamma, \beta\gamma, \gamma^2, \alpha, \beta, \gamma\}$;
\item $2$ polynomials $g_1$ and $g_2$.
\end{itemize}

The $84\times 84$ left upper submatrix of~$\hat A$ is upper-triangular with $1$'s on its main diagonal. This allows one to further reduce the size of the elimination template to $37\times 81$ similarly as it was done for the 4-point solver, cf. Eq~\eqref{eq:newmatrix}. Then all possible solutions to the problem are found from the eigenvectors of a $44\times 44$ action matrix.

Note that the 6-point algorithm by Stew{\'e}nius et al.~\cite{SNOA} requires the G-J elimination on a $60\times 120$ matrix and then more computation is needed to complete the Gr\"{o}bner basis construction. Besides, the algorithm produces up to 64 solutions. Our new 5-point algorithm produces the entire Gr\"{o}bner basis by a single G-J elimination on a $37\times 81$ matrix, and the number of possible solutions is only~44.

\section{Experiments on Synthetic Data}
\label{sec:synth}

In this section, we test our minimal solvers on synthetic data. The data setup is given in the following table:
\begin{center}
\bigskip\begin{tabular}{|l|c|}
\hline
Distance to the scene & 1\\\hline
Scene depth & 0.5\\\hline
Baseline length & 0.1\\\hline
Image dimensions & $752 \times 480$ (WVGA) \\\hline
Field of view & 60 degrees\\\hline
\end{tabular}\bigskip
\end{center}

The generalized cameras are modelled as multi-cameras with five and six optical centers for the 5- and 6-point solvers respectively. The optical centers are randomly generated inside two balls of radius $0.05$ each. The centers of the balls are related by the $(R, t)$ transformation. Each point is viewed from a single optical center.

In the sequel, we use the following abbreviations for the minimal algorithms:
\begin{center}
\bigskip\begin{tabular}{|l|c|}
\hline
4pt+angle & our 4-point solver\\\hline
4pt+angle nm\footnotemark & numerical 4-point solver from~\cite{LHLP}\\\hline
5pt & 5-point solver from~\cite{SEN06}\\\hline
gen. 5pt+angle & our 5-point solver \\\hline
gen. 6pt & 6-point solver from~\cite{SNOA}\\\hline
\end{tabular}\bigskip
\end{center}

\footnotetext{The MATLAB implementation of the ``4pt+angle nm'' algorithm was translated from the publicly available C++ code. This solver was only used in the speed comparison experiment.}

\subsection{Numerical Accuracy}

The numerical error is measured by the value
\begin{equation}
\min\limits_i\|R_i - \bar R\|,
\end{equation}
where $i$ counts all real solutions and $\bar R$ is the ground truth relative rotation matrix. The numerical error distributions are reported in Fig.~\ref{fig:ErrNum_timings} (left). The median errors are: $5.10\times 10^{-13}$ (4pt+angle), $1.35\times 10^{-14}$ (5pt), $4.86\times 10^{-11}$ (gen. 5pt+angle), $4.93\times 10^{-9}$ (gen. 6pt).

\subsection{Timings}
\label{ssec:timings}

We compared the speed of the algorithms. The mean running times for two major steps of the solvers are presented in Fig.~\ref{fig:ErrNum_timings} (right). The computations were performed on a system with 2.3~GHz processor. Note that the generalized-camera solvers are in order of magnitude slower. For the sake of visibility, their timings are divided by~$10$ on the figure. The mean total timings are: $0.75$~ms (4pt+angle), $1.3$~ms (4pt+angle nm), $0.49$~ms (5pt), $5.9$~ms (gen. 5pt+angle), $11.5$~ms (gen. 6pt).

It is interesting to note that the Gr\"{o}bner basis computation for our 4-point solution is performed even faster than for the 5-point algorithm. However, since the action matrix for the 5-point algorithm is four times smaller than for the 4-point problem, the eigendecomposition step for the 5pt solver is performed much faster and the total running time for it is about $1.5$ times less.

\begin{figure}[t]
\centering
\includegraphics[width=1\hsize]{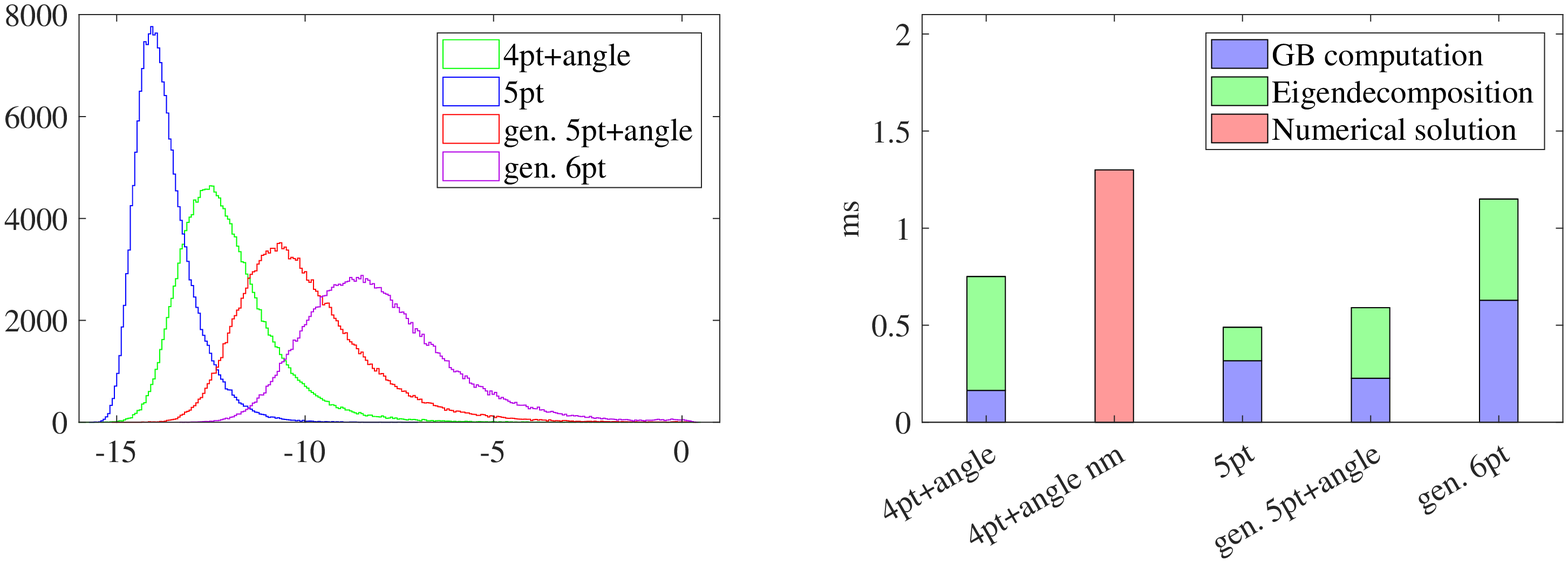}
\caption{$\log_{10}$ of numerical error (left) and mean speed (right) for different minimal solvers. The timings for generalized-camera solvers are divided by 10 for visibility}\label{fig:ErrNum_timings}
\end{figure}

\subsection{Performance Under Noise}

To evaluate the robustness of our solvers, we added two types of noise to the initial data. First, the image noise which is modelled as zero-mean, Gaussian distributed with a varying standard deviation. Second, the angle noise resulting from inaccurately found relative rotation angle~$\theta$. Following~\cite{LHLP}, we modelled that noise as $\theta s$, where~$s$ has the Gaussian distribution with zero mean and standard deviation~$\sigma$. In our experiments, $\sigma$ takes only two values: $0$ and~$0.05$.

The rotational and translational errors are presented in Fig.~\ref{fig:mR} and Fig.~\ref{fig:mT} respectively. Fig.~\ref{fig:mS} demonstrates the relative errors in translation scale for the generalized-camera solvers. Each point on the plots is a lower quartile of $10^4$ trials.

\begin{figure}[t]
\centering
\includegraphics[width=1\hsize]{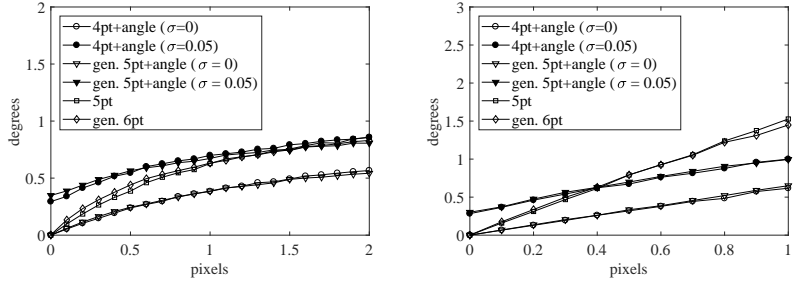}
\caption{Rotational errors against image noise. Left: forward motion, right: sideways motion}\label{fig:mR}
\end{figure}

\begin{figure}[t]
\centering
\includegraphics[width=1\hsize]{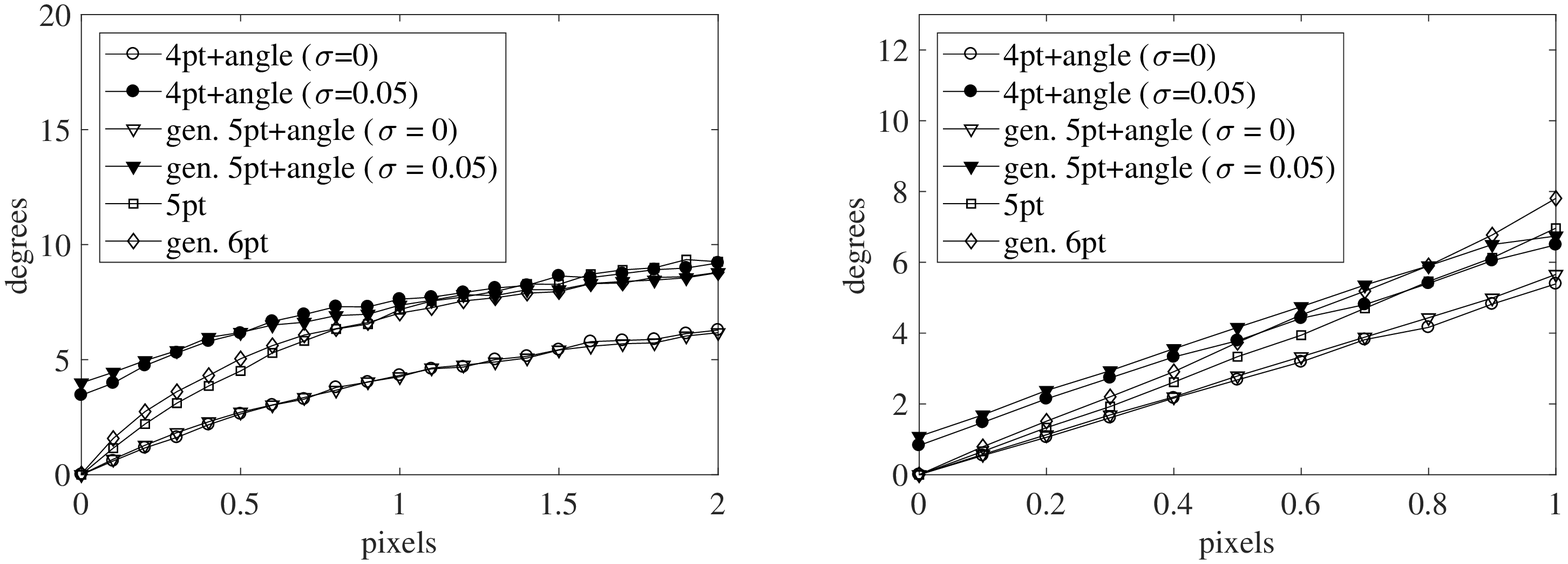}
\caption{Translational errors against image noise. Left: forward motion, right: sideways motion}\label{fig:mT}
\end{figure}

\begin{figure}[t]
\centering
\includegraphics[width=1\hsize]{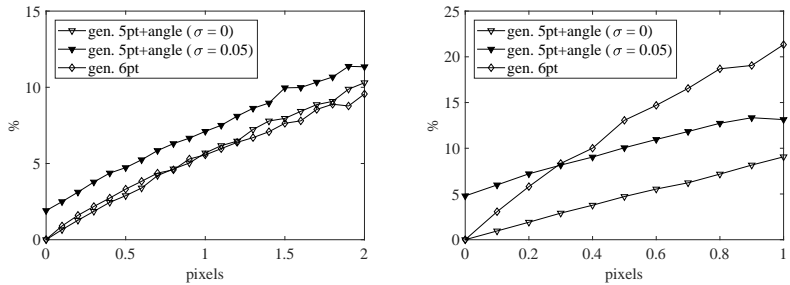}
\caption{Relative errors in translation scale against image noise. Left: forward motion, right: sideways motion}\label{fig:mS}
\end{figure}

\subsection{Performance in Presence of Outliers}

To ensure the robustness of a minimal solver in presence of outliers (incorrect matches), it is a common practice to use a hypothesize-and-test architecture such as Random Sample Consensus (RANSAC)~\cite{FB}. The RANSAC framework consists in the following major steps. First, a number of minimal subsets of the data (samples) are randomly chosen from the entire set of observations. Second, a minimal solver is used for each sample to generate a set of hypotheses. Third, each hypothesis is subsequently scored by the scoring function saying if an observation is an inlier or outlier. Finally, the winner is the hypothesis supported by the maximal number of inliers.

In our experiments, the hypotheses for regular cameras were scored by the Sampson approximation error~\cite{HZ}, whereas the hypotheses for generalized cameras were scored by the reprojection error. The number of observations was set to 100 and the fraction of outliers was set to~30\%. The results are presented in Fig.~\ref{fig:RANSACmTS}. Each point on the plots is a mean of 100 trials.

\begin{figure}[t]
\centering
\includegraphics[width=1\hsize]{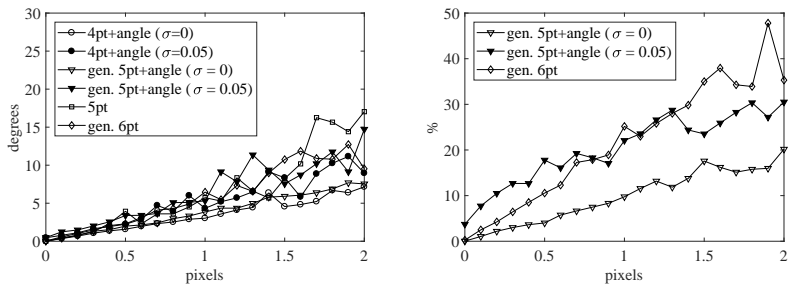}
\caption{Translational errors (left) and relative errors in translation scale (right) versus increasing image noise. Sideways motion, 100 observations and 30\% of outliers}\label{fig:RANSACmTS}
\end{figure}

\section{Experiments on Real Image Sequences}
\label{sec:real}

In this section, we test the algorithms on the publicly available EuRoC dataset~\cite{EuRoC} containing synchronized sequences of the data recorded from an IMU and two cameras on board a micro-aerial vehicle. Specifically, we test the solvers on the ``Machine Hall~01'' and ``Vicon Room~01'' datasets. For the 4pt+angle and 5-pt algorithms, the images taken by camera~``0'' were only used. For the generalized-camera solvers, the images from both cameras were involved.

The relative rotation angle is derived from the readings of a triple-axis gyroscope as follows. The gyroscope reading at time~$\xi_i$ is an angular rate 3-vector~$w_i$. Let $\Delta \xi_i = \xi_i - \xi_{i - 1}$, where $i = 1, \ldots, n$. Then the relative rotation matrix~$R_n$ between the $0$th and $n$th frames is approximately found from the recursion
\begin{equation}
R_i = \exp([w_i]_\times \Delta \xi_i) R_{i - 1},
\end{equation}
where $R_0 = I$. The relative rotation angle is then equal to $\arccos\frac{\tr R_n - 1}{2}$.

To prevent the baseline between two consecutive images from being too small and thus to make the computation for regular cameras more stable, we reduced the number of images by taking only each 4th image in the sequences. We also cropped the initialization tests in the beginning of ``Machine Hall~01'' dataset. The resulting sequences consist in total of 683 images for each dataset.

After that, we computed the rotational and translational errors for each pair of consecutive frames using the ground truth. The results are presented as box plots in Fig.~\ref{fig:MH01_mRT} and Fig.~\ref{fig:VR01_mRT}. The middle line represents the median and the blue box shows values from lower (25\%) to upper (75\%) quartile. The red markers indicate outliers. To make the comparison as objective as possible, we did not apply any iterative refinement during the estimation.

\begin{figure}[t]
\centering
\includegraphics[width=1\hsize]{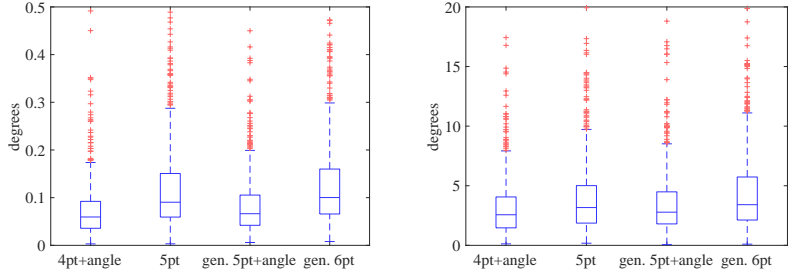}
\caption{``Machine Hall~01'' dataset. Rotational (left) and translational (right) errors for different solvers}\label{fig:MH01_mRT}
\end{figure}

\begin{figure}[t]
\centering
\includegraphics[width=1\hsize]{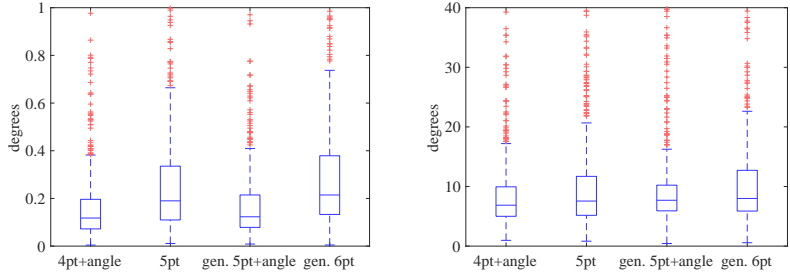}
\caption{``Vicon Room~01'' dataset. Rotational (left) and translational (right) errors for different solvers}\label{fig:VR01_mRT}
\end{figure}

\section{Discussion of Results}
\label{sec:disc}

The main contribution of this paper is the introduction of two new minimal algorithms for the relative pose estimation problem: the first one is for regular calibrated cameras, and the second one is for generalized calibrated cameras. Both algorithms operate with several point correspondences in two views and also with the known relative rotation angle. The first algorithm is a significant improvement of the existing solution, whereas the second algorithm for generalized cameras is novel.

Our solutions have relatively small elimination templates for the Gr\"{o}bner basis computation: $16\times 36$ for regular cameras and $37\times 81$ for generalized cameras. In a series of experiments on simulated data, we demonstrated that our solutions are comparable or even faster than the state-of-the-art solvers. We also verify their numerical accuracy for noise-free data and robustness for noisy data.

It is shown on synthetic and real datasets that, being encapsulated into a hy\-pothesize-and-test architecture such as RANSAC, our solutions become robust in presence of outliers and again demonstrate better performance than the existing solvers.

Finally, both our solutions use the same block transformation of the elimination template to significantly simplify the Gr\"{o}bner basis construction. We expect that this approach might be applied to some other computer vision problems described in terms of polynomial equations and containing a low-degree constraint like in Eq.~\eqref{eq:quadconstr}.

\subsection*{Acknowledgements}

The work of E.M. was supported by Act 211 Government of the Russian Federation, contract No.~02.A03.21.0011.

\appendix
\section{Proof of Lemma~\ref{lem:FF}}
\label{app:A}

\setcounter{lemma}{0}

\begin{lemma}
Let $F_{ijk}$ be a matrix defined in~\eqref{eq:Fijk}. Then the following identity holds
\[
\det F_{ijk} = \det F_{jki}.
\]
\end{lemma}

\begin{proof}
First, recall that
\begin{equation}
R (x\times y) = (Rx)\times (Ry)
\end{equation}
for any 3-vectors~$x, y$ and $R \in \mathrm{SO}(3)$. We also need the following formula
\begin{equation}
\label{eq:auxformula}
z_1^{\mathrm T}(x_1\times y_1) \cdot z_2^{\mathrm T}(x_2\times y_2) = \det \left(\begin{bmatrix} x_1^{\mathrm T} \\ y_1^{\mathrm T} \\ z_1^{\mathrm T} \end{bmatrix} \begin{bmatrix} x_2 & y_2 & z_2 \end{bmatrix} \right),
\end{equation}
which evidently holds for any 3-vectors $x_i$, $y_i$ and~$z_i$.

Let us denote $a = R q'_i$, $b = R q'_j$, $c = R q'_k$, $d = q''_i$, $e = q''_j$, $f = q''_k$. Then, it follows from the definition of matrix~$F_{ijk}$ that
\begin{equation}
F_{ijk} = \begin{bmatrix} e^{\mathrm T}(a\times b) & b^{\mathrm T}(d\times e) \\ f^{\mathrm T}(a\times c) & c^{\mathrm T}(d\times f) \end{bmatrix}
\end{equation}
and
\begin{equation}
F_{jki} = \begin{bmatrix} f^{\mathrm T}(b\times c) & c^{\mathrm T}(e\times f) \\ d^{\mathrm T}(b\times a) & a^{\mathrm T}(e\times d) \end{bmatrix}.
\end{equation}
Denote for short by $xy$ the scalar product of vectors~$x$ and~$y$. Then, equality~\eqref{eq:FF} together with formula~\eqref{eq:auxformula} assume that
\begin{multline}
\det \begin{bmatrix} ad & af & ac \\ bd & bf & bc \\ de & ef & ce \end{bmatrix} - \det \begin{bmatrix} ad & cd & df \\ ae & ce & ef \\ ab & bc & bf \end{bmatrix} \\= \det \begin{bmatrix} be & bd & ab \\ ce & cd & ac \\ ef & df & af \end{bmatrix} - \det \begin{bmatrix} be & ae & de \\ bf & af & df \\ bc & ac & cd \end{bmatrix}.
\end{multline}
The latter identity is verified directly. Lemma~\ref{lem:FF} is proved.
\end{proof}

\bibliographystyle{amsplain}

\end{document}